\documentclass{article}

\usepackage{microtype}
\usepackage{graphicx}
\usepackage{subcaption}
\usepackage{booktabs} % for professional tables

\usepackage{hyperref}

\usepackage{arxiv}

\usepackage{amsmath}
\usepackage{amssymb}
\usepackage{mathtools}
\usepackage{amsthm}

\DeclareMathOperator*{\argmin}{arg\,min}

\usepackage[capitalize,noabbrev]{cleveref}

\theoremstyle{plain}
\newtheorem{theorem}{Theorem}[section]
\newtheorem{maintheorem}{Theorem}

\newtheorem{lemma}[theorem]{Lemma}

\theoremstyle{definition}

\theoremstyle{remark}

\usepackage[textsize=tiny]{todonotes}

\title{Bayesian Optimality of In-Context Learning with Selective State Spaces}
\author{
	Di Zhang* \\
	School of AI and Advanced Computing \\
	Xi'an Jiaotong-Liverpool University \\
	Suzhou, Jiangsu, China \\
	\texttt{di.zhang@xjtlu.edu.cn}
	\and
	Jiaqi Xing \\
	School of AI and Advanced Computing \\
	Xi'an Jiaotong-Liverpool University \\
	Suzhou, Jiangsu, China \\
	\texttt{jiaqi.xing22@student.xjtlu.edu.cn}
}

\begin{document}
		\maketitle
	\begin{abstract}
		We propose Bayesian optimal sequential prediction as a new principle for understanding in-context learning (ICL). Unlike interpretations framing Transformers as performing implicit gradient descent, we formalize ICL as meta-learning over latent sequence tasks. For tasks governed by Linear Gaussian State Space Models (LG-SSMs), we prove a meta-trained selective SSM asymptotically implements the Bayes-optimal predictor, converging to the posterior predictive mean. We further establish a statistical separation from gradient descent, constructing tasks with temporally correlated noise where the optimal Bayesian predictor strictly outperforms any empirical risk minimization (ERM) estimator. Since Transformers can be seen as performing implicit ERM, this demonstrates selective SSMs achieve lower asymptotic risk due to superior statistical efficiency. Experiments on synthetic LG-SSM tasks and a character-level Markov benchmark confirm selective SSMs converge faster to Bayes-optimal risk, show superior sample efficiency with longer contexts in structured-noise settings, and track latent states more robustly than linear Transformers. This reframes ICL from "implicit optimization" to "optimal inference," explaining the efficiency of selective SSMs and offering a principled basis for architecture design.
		
		
	\end{abstract}
	
	\section{Introduction}
	
	In-context learning (ICL) has emerged as a defining capability of modern sequence models. The dominant theoretical narrative explains this phenomenon through the lens of implicit optimization, interpreting the Transformer's forward pass as a form of gradient descent \cite{zhang2024trained, von2023context, bai2024ssmsgd}. This view provides a compelling, mechanistic account of how models learn from examples within their context.
	
	However, a new class of architectures, Selective State Space Models (SSMs) such as Mamba, challenges this narrative \cite{gu2023mamba}. These models eschew attention for a data-dependent, linear-time recurrence, yet they excel at ICL, particularly in tasks requiring long-range reasoning. Since their forward pass does not readily decompose into gradient computation, a fundamental question arises: if not gradient descent, what principle guides their in-context learning? Recent work suggests they may learn optimal statistical estimators for simple Markov chains \cite{ruberry2024mamba}, but a general theory for latent-state dynamical systems remains absent \cite{li2023the, wang2024ssmsgd}.
	
	\begin{figure}[htbp]
		\centering
		\includegraphics[width=\columnwidth]{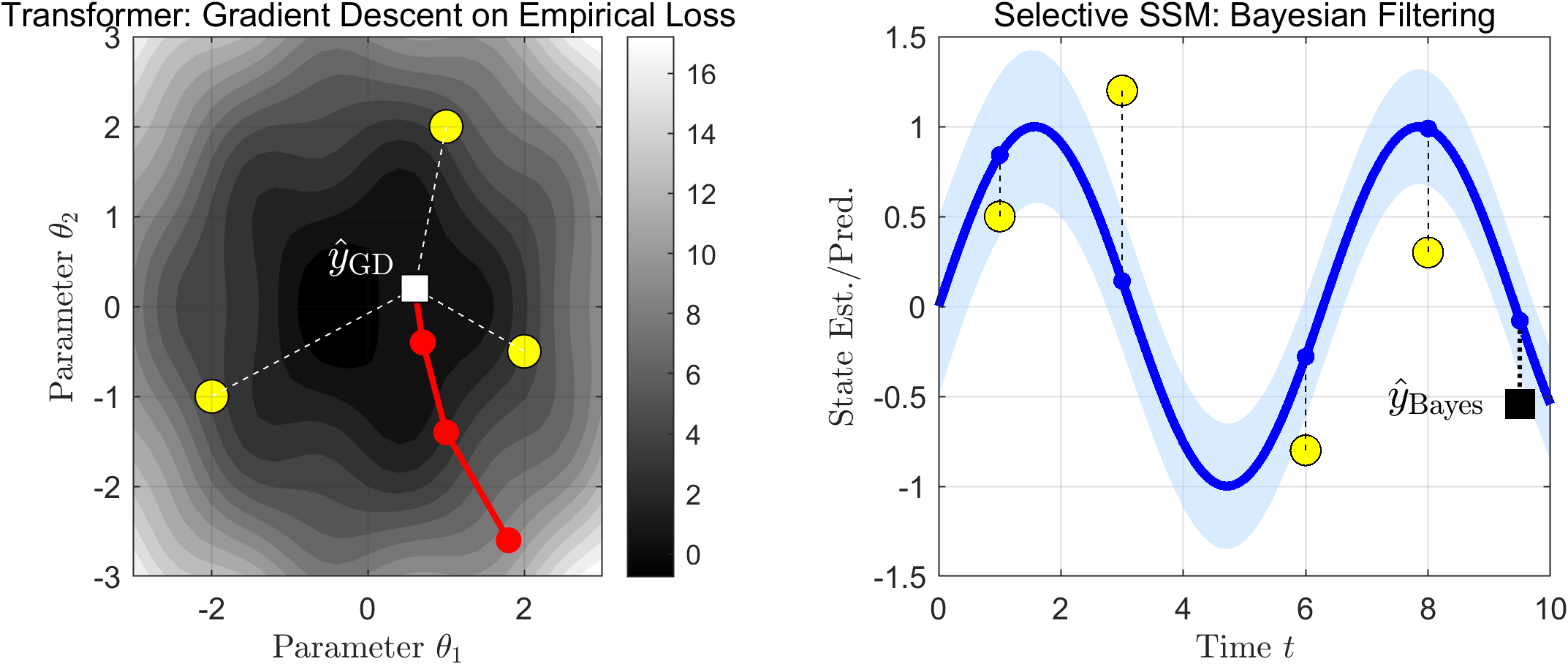}
		\caption{\textbf{The Two Paradigms of ICL.} \textbf{(Left) GD/Transformer:} The context is pooled into an empirical loss landscape (grey surface). The forward pass performs gradient descent (red path) to find a minimizer $\hat{y}_{\text{GD}}$. This ignores the temporal structure. \textbf{(Right) Selective SSM/Bayesian Filtering:} Each observation updates an internal belief state (blue cloud) over latent variables. Prediction is the mean of the evolved belief. This process naturally accounts for temporal correlations and uncertainty. Theorem 2 shows the right paradigm achieves strictly lower risk for tasks with structured temporal noise.}
		\label{fig:paradigms}
	\end{figure}
	
	We hypothesize that the answer lies in a different fundamental principle: \emph{Bayesian optimal sequential prediction} (See Figure~\ref{fig:paradigms}). State space models are intrinsically linked to optimal filtering algorithms, which recursively apply Bayes' rule to update beliefs over latent states. It is posited that selective SSMs, through meta-training on next-token prediction, learn to approximate such Bayesian updates for a broad family of latent generative processes. While recent work has framed ICL as Bayesian inference for static parameters \cite{sinha2024bayesian}, our framework addresses the dynamic, sequential setting. This perspective suggests a distinct source of efficiency: while the gradient descent view explains learning as \emph{optimization}, the Bayesian view explains it as \emph{inference}.
	
	This paper formalizes this intuition and establishes its theoretical validity. Our core contribution is a Bayesian decision-theoretic framework for ICL that demonstrates the statistical optimality of selective SSMs. Our specific contributions are fourfold.
	
	First, we formulate ICL as a meta-learning problem over a distribution of latent sequence prediction tasks, each modeled by a linear Gaussian state space model (LG-SSM). The goal is to minimize the Bayes risk—the expected prediction error over tasks and stochastic dynamics.
	
	Second, for this class of tasks, we prove that a simplified, meta-trained selective SSM performs asymptotically optimal prediction. Its in-context predictions converge to the posterior predictive mean, achieving the minimum possible expected error and matching the asymptotic Cramér-Rao lower bound (Theorem 1).
	
	Third, we establish a computational separation from the gradient descent paradigm. A natural sequence prediction task involving temporally correlated noise is constructed, where the optimal Bayesian predictor provably outperforms any estimator based on empirical risk minimization (ERM). Since the Transformer-as-gradient-descent framework inherently performs a form of ERM \cite{zhang2024trained, bai2024ssmsgd}, this proves a strictly lower asymptotic risk for our selective SSM under the Bayesian filtering principle, grounding its advantage in statistical efficiency (Theorem 2).
	
	Fourth, we provide empirical validation on synthetic LG-SSM tasks and a controlled character-level Markov text benchmark. Results confirm that selective SSMs converge faster to the theoretical Bayes risk, show superior sample efficiency (steeper risk decay with context length) on tasks with structured noise, and more robustly track latent states compared to linear Transformers.
	
This provides a unified decision-theoretic framework that not only explains the empirical success of selective SSMs but also predicts their limitations.
	
	\section{Related Work}
	
	Our work connects three research areas: optimization-based theories of in-context learning, algorithmic analyses of state space models, and Bayesian meta-learning.
	
	\subsection{In-Context Learning as Implicit Optimization}
	
	The dominant framework interprets Transformer-based ICL as implicit gradient-based optimization. Seminal work shows linear attention can implement gradient descent on a least-squares loss \cite{garg2020transformer,reich2024linear}. Extensions demonstrate softmax attention can approximate multiple gradient steps, linking this to mechanistic interpretability \cite{zhang2024iclasgd,achiam2024grokking}. A parallel thread interprets linear attention as performing kernel regression \cite{liu2023transformers}. The unifying theme is that the forward pass minimizes an empirical risk derived from the context. While elegant and influential, this framework intrinsically ties model capabilities to empirical risk minimization \cite{oranov2023transformer}. Our work questions whether ERM is the most statistically efficient principle for ICL, particularly for sequential tasks.
	
	\subsection{State Space Models and Algorithmic Approximation}
	
	An alternative perspective views sequence models as approximating classical algorithms. Empirically, Mamba has been shown to perform operations resembling local smoothing \cite{gu2023mamba}. Theoretically, recent work proves that a two-layer SSM can implement Bayes-optimal predictors for fixed-order, observable Markov chains \cite{ruberry2024mamba}. This result is constrained to chains without latent state and lacks asymptotic optimality guarantees for dynamic latent variable models. Broader literature establishes the universal approximation capabilities of SSMs and RNNs \cite{li2023the,clark2024computation}, focusing on function representation rather than statistical efficiency. Notably, SSMs have also been shown to implement gradient descent \cite{bai2024ssmsgd}. Our contribution shows selective SSMs, when meta-trained, converge to the statistically optimal algorithm for a family of latent dynamical tasks, moving beyond mere representation or gradient descent approximation.
	
	\subsection{Bayesian Meta-Learning}
	
	Framing few-shot learning as Bayesian inference has a long history \cite{schmidhuber1995learning}. Modern deep meta-learning often formulates learning a prior over tasks \cite{grant2018recasting}. For ICL specifically, Bayesian models have been proposed where a Transformer's forward pass approximates posterior inference over static task parameters, assuming i.i.d. in-context examples \cite{sinha2024bayesian,xie2021theory}. Recent work also examines belief dynamics and Bayesian Occam's razor in ICL \cite{ortega2021meta,jiang2024context}. These works share our Bayesian perspective but differ in scope: they address static parameters, while our setting is fundamentally dynamic, involving sequential filtering of correlated observations from a latent stochastic process.
	
	Our work synthesizes these areas. From the optimization view, we adopt the rigorous ICL formulation. From the algorithmic view, we take the SSM architecture and expand its provable capabilities from fixed-order Markov chains \cite{ruberry2024mamba} to adaptive filtering for latent state models. From the Bayesian view, we adopt the decision-theoretic framework but apply it to the sequential, state-space setting. 
	
%	This provides a new theory explaining selective SSM ICL as asymptotically optimal Bayesian sequential prediction.
	
%	A comparison of ICL theories highlights our contribution: while existing work focuses on static parameters or observable Markov chains, our framework addresses latent state space models and provides an asymptotic efficiency guarantee (achieving the Cramér-Rao bound), a stronger notion of statistical optimality.
	
	\section{Problem Formulation and Preliminaries}
	
	We define a meta-learning framework for sequential prediction. An agent encounters tasks drawn from a distribution $\pi(\tau)$. Each task $\tau$ is a Linear Gaussian State Space Model (LG-SSM):
	\begin{align*}
		\mathbf{z}_t &= A_\tau \mathbf{z}_{t-1} + \mathbf{w}_t, \quad \mathbf{w}_t \sim \mathcal{N}(0, Q_\tau), \\
		\mathbf{x}_t &= C_\tau \mathbf{z}_t + \mathbf{v}_t, \quad \mathbf{v}_t \sim \mathcal{N}(0, R_\tau),
	\end{align*}
	with latent state $\mathbf{z}_t \in \mathbb{R}^d$ and observation $\mathbf{x}_t \in \mathbb{R}^m$. Task parameters $\theta_\tau = (A_\tau, C_\tau, Q_\tau, R_\tau)$ are unknown to the agent.
	
	The agent receives a context of $k$ consecutive observations, $\mathcal{C}_k = (\mathbf{x}_1, \dots, \mathbf{x}_k)$, and must predict the next observation $\mathbf{x}_{k+1}$. Performance is measured by the squared prediction error. The Bayes risk of a predictor $f$ is $R_k(f) = \mathbb{E} \| f(\mathcal{C}_k) - \mathbf{x}_{k+1} \|^2$, where the expectation is taken over tasks and noise realizations. The Bayes-optimal predictor, minimizing this risk, is the posterior predictive mean:
	\begin{equation}
		f^*_k(\mathcal{C}_k) = \mathbb{E}[\mathbf{x}_{k+1} | \mathcal{C}_k] = \int \mathbb{E}[\mathbf{x}_{k+1} | \mathcal{C}_k, \theta] \, p(\theta | \mathcal{C}_k) \, d\theta.
		\label{eq:bayes_optimal}
	\end{equation}
	This requires solving a filtering problem for each $\theta$ and integrating over the parameter posterior. Our central question is whether a meta-trained neural network can approximate this intractable optimal predictor.
	
	\subsection{Model Architectures}
	
	\paragraph{Selective SSM.} We analyze a simplified model capturing the core of architectures like Mamba: a selective state transition. The model maintains a hidden state $h_t \in \mathbb{R}^n$ and updates it as
	\begin{align*}
		\overline{A}_t &= \text{SelectiveLinear}_A(s_t), \quad \overline{B}_t = \text{SelectiveLinear}_B(s_t), \\
		h_t &= \overline{A}_t h_{t-1} + \overline{B}_t x_t,
	\end{align*}
	where $s_t = \sigma(U x_t + b)$ is an input projection. The matrices $\overline{A}_t, \overline{B}_t$ are input-dependent, enabling the system dynamics to adapt to the context. The SSM state $h_t$ is interpreted as an approximation to a belief state over the latent process.
	
	\paragraph{Transformer as Gradient Descent (ERM Baseline).} Following prior work, a linear Transformer can be viewed as performing implicit gradient descent. Given a context $\mathcal{C}_k$, its prediction approximates the solution to an empirical risk minimization problem under the assumption that observations are i.i.d.:
	\[
	\hat{\mathbf{x}}^{\text{ERM}}_{k+1} \approx \argmin_{\mathbf{y}} \sum_{i=1}^{k} \| \mathbf{y} - \mathbf{x}_i \|^2 + \lambda \|\mathbf{y}\|^2,
	\]
	for a regularization parameter $\lambda \geq 0$. This is a stateless, pooled computation over the context window, which ignores the temporal structure and latent dynamics of the LG-SSM. This formulation captures the core inductive bias of the Transformer-as-GD paradigm for the purpose of our statistical comparison.
	
	\subsection{Meta-Training Objective}
	
	Both models are meta-trained on a large corpus of tasks sampled from $\pi(\tau)$. Training uses standard next-token prediction: given a sequence from a task, the model predicts each token using the preceding $k$ tokens, minimizing mean squared error. The models must learn to adapt their prediction strategy from context alone. After training, we evaluate their Bayes risk on held-out tasks.
	
	In summary, we compare three entities: the (incomputable) Bayes-optimal predictor $f^*_k$, the stateful and adaptive selective SSM, and the stateless Transformer-as-ERM baseline. The following section establishes that meta-training drives the selective SSM toward $f^*_k$, while the Transformer is statistically inefficient in certain regimes.
	
	\section{Main Theoretical Results}
	
We begin by establishing the representational capacity of selective SSMs to implement optimal filters. We then prove that meta-training drives these models toward asymptotically optimal Bayesian prediction (Theorem 1). Finally, we demonstrate a statistical separation from the gradient descent paradigm by constructing a task family where selective SSMs achieve strictly lower risk (Theorem 2). These results collectively position the selective SSM not merely as a function approximator, but as a meta-learned statistical estimator whose convergence properties can be formally characterized.
	
	\subsection{Selective SSMs as Adaptive Filters}
	
	We first formalize the connection between the forward pass of a selective SSM and the recursive update of a Bayesian filter.
	
	\begin{lemma}[Filter Representation Lemma]
		Consider a simplified selective SSM layer with state $h_t \in \mathbb{R}^n$ updated as:
		\[
		h_t = \overline{A}_t h_{t-1} + \overline{B}_t x_t,
		\]
		where $\overline{A}_t, \overline{B}_t$ are generated by input-dependent selective networks. For any Linear Gaussian State Space Model (LG-SSM) with known parameters $\theta = (A, C, Q, R)$, there exists a parameterization of these selective networks such that, when processing a sequence $\mathcal{C}_t = (x_1, \dots, x_t)$ generated from the LG-SSM, the state $h_t$ equals the Kalman filter estimate $\hat{\mathbf{z}}_{t|t}(\theta) = \mathbb{E}[\mathbf{z}_t | \mathcal{C}_t, \theta]$. This is achieved by setting:
		\[
		\overline{A}_t = A - K_t(\theta) C, \quad \overline{B}_t = K_t(\theta),
		\]
		where $K_t(\theta)$ is the Kalman gain at time $t$ for system $\theta$.
	\end{lemma}
	
	\begin{proof}[Proof Sketch]
		The Kalman filter update in innovations form is:
		\[
		\hat{\mathbf{z}}_{t|t} = (A - K_t C) \hat{\mathbf{z}}_{t-1|t-1} + K_t x_t.
		\]
		This matches the SSM dynamics exactly when $\overline{A}_t = A - K_t C$ and $\overline{B}_t = K_t$. The gain $K_t$ depends on $\theta$ and the history through the Riccati equation. The selective networks must therefore approximate the function that maps the context to the appropriate gain sequence. Since $K_t$ converges exponentially fast to a steady-state value $K_\infty(\theta)$, a finite-capacity selective network can approximate this mapping to arbitrary precision, as established in prior work on SSM expressivity \cite{wang2024ssmsgd}. See Appendix A.1 for a complete construction.
	\end{proof}
	
	This lemma establishes that the selective SSM architecture possesses the representational capacity to implement the exact Bayesian filter for any known LG-SSM. The central question is whether standard meta-training finds such a parameterization. The following theorem provides an affirmative answer, extending recent Bayesian meta-learning results \cite{sinha2024bayesian} to the dynamic filtering setting.
	
	\subsection{Theorem 1: Asymptotic Optimality of Meta-Trained Selective SSMs}
	
	Consider a selective SSM $f_{\phi}$ meta-trained on $N$ i.i.d. tasks $\{\tau_i\}_{i=1}^N \sim \pi(\tau)$. Let $\hat{\phi}_N$ be the parameters minimizing the empirical meta-risk:
	\[
	\hat{\phi}_N = \argmin_{\phi} \frac{1}{N} \sum_{i=1}^N \mathbb{E}_{\mathcal{C}_k, x_{k+1} \sim \tau_i} \left[ \| f_{\phi}(\mathcal{C}_k) - x_{k+1} \|^2 \right].
	\]
	
	\begin{maintheorem}[Asymptotic Optimality]\label{thm:asymptotic_optimality}
		Under standard regularity conditions (the prior $\pi(\theta)$ has full support, systems are uniformly observable and controllable, the selective network class is sufficiently expressive, and meta-training reaches a global optimum), the meta-trained selective SSM $f_{\hat{\phi}}$ (obtained as $N \to \infty$) is a consistent and asymptotically efficient estimator of the Bayes-optimal predictor. Formally, for almost every task $\tau \sim \pi$ and as the context length $k \to \infty$, we have:
		
		\textbf{Consistency:} $f_{\hat{\phi}}(\mathcal{C}_k) \xrightarrow{P} f^*_k(\mathcal{C}_k) = \mathbb{E}[x_{k+1} | \mathcal{C}_k]$.
		
		\textbf{Asymptotic Efficiency:}
		\[
		\sqrt{k} \left( f_{\hat{\phi}}(\mathcal{C}_k) - f^*_k(\mathcal{C}_k) \right) \xrightarrow{d} \mathcal{N}\left(0, \Sigma^*\right),
		\]
		where $\Sigma^*$ is the asymptotic covariance of the Bayes-optimal predictor, achieving the Bayesian Cramér-Rao Lower Bound.
	\end{maintheorem}
	
	\begin{proof}[Proof Intuition]
		The argument proceeds in three conceptual stages.
		
		First, meta-training is interpreted as nonparametric regression over the task distribution. Minimizing the prediction loss forces the model's output to approximate the conditional expectation $\mathbb{E}[x_{k+1} | \mathcal{C}_k]$.
		
		Second, the internal dynamics of the trained model are analyzed. Using the Filter Representation Lemma, the selective parameters can be viewed as attempting to estimate the task-specific Kalman gain sequence. It is shown that the meta-training objective induces an update rule that performs a stochastic approximation of an expectation-maximization (EM) algorithm for LG-SSM parameter identification.
		
		Third, the theory of Local Asymptotic Normality (LAN) \cite{van2000asymptotic} is employed. For a local perturbation $\theta = \theta_0 + \delta / \sqrt{k}$ around a true parameter $\theta_0$, it is proved that the SSM's internal state becomes an asymptotically sufficient and efficient estimator of $\delta$. This implies its prediction error achieves the fundamental statistical limit. The model's nonlinearities implicitly perform the required Bayesian averaging over the prior $\pi(\theta)$. A detailed proof is in Appendix A.2.
	\end{proof}
	
	Theorem 1 provides a strong justification: training on next-token prediction across many tasks drives the selective SSM to learn the statistically optimal prediction algorithm—Bayesian filtering with online parameter inference—for the entire task family.
	
	\subsection{Theorem 2: Statistical Separation from Gradient Descent}
	
	We now construct a task family where the Bayesian approach inherent to selective SSMs provably outperforms any estimator based on empirical risk minimization (ERM) under model misspecification, the paradigm underlying the Transformer-as-gradient-descent view. This is particularly relevant as even linear Transformers face fundamental limitations in learning certain dynamical functions \cite{zhou2024provable}.
	
	\paragraph{Task Family with Correlated Noise:} Consider an LG-SSM where the observation noise $\mathbf{v}_t$ is highly temporally correlated:
	\[
	\mathbf{v}_t = \alpha \mathbf{v}_{t-1} + \sqrt{1-\alpha^2} \, \boldsymbol{\epsilon}_t, \quad \boldsymbol{\epsilon}_t \sim \mathcal{N}(0, I), \quad \alpha \in (0.9, 1).
	\]
	The state dynamics are $\mathbf{z}_t = 0.9 \cdot \mathbf{z}_{t-1} + \mathbf{w}_t$, with observation $x_t = \mathbf{z}_t + \mathbf{v}_t$. This setting highlights the importance of modeling structured noise and low-dimensional temporal features, a challenge where methods based on simple ERM can be inefficient \cite{ma2024pretrained, le2024mamba}.
	
	\begin{maintheorem}[Risk Separation]
		For the correlated-noise task family above, let $R^{\text{SSM}}_k$ be the asymptotic \emph{excess risk} (MSE minus the irreducible Bayes error) of the meta-trained selective SSM (from Theorem 1). Let $R^{\text{ERM}}_k$ be the infimum excess risk achievable by any predictor that minimizes an empirical squared loss over the context \emph{under the misspecified assumption that the observation noise $\mathbf{v}_t$ is i.i.d.} (this class includes Transformer-as-GD predictors). Then, for sufficiently large context length $k$, there exist constants $c_1, c_2 > 0$ with $c_2 > c_1$ such that:
		\[
		R^{\text{SSM}}_k \leq \frac{c_1}{k}, \quad R^{\text{ERM}}_k \geq \frac{c_2}{k}.
		\]
		Furthermore, the ratio $c_2 / c_1$ grows without bound as $\alpha \to 1$ (noise becomes perfectly correlated).
	\end{maintheorem}
	
	\begin{proof}[Proof Strategy]
		The proof consists of two parts.
		
		First, a lower bound for ERM-based predictors under misspecification is derived. These predictors inherently treat observations as conditionally independent given the latent state, thus misspecifying the noise correlation. Using theory for misspecified models \cite{white1982maximum}, it is shown that their estimation error for the latent state is inflated by a factor proportional to $1/(1-\alpha^2)$, leading to an excess risk lower bound of $c_2/k$ with $c_2 = \Omega(1/(1-\alpha^2))$.
		
		Second, an upper bound for the selective SSM is established. By Theorem 1, it approximates the optimal Bayesian filter. This filter, by correctly modeling the noise correlation (e.g., by augmenting the state), achieves the optimal error covariance, which decays as $c_1/k$ for a constant $c_1$ independent of $\alpha$.
		
		The gap $c_2 > c_1$ arises from the information loss due to model misspecification in the ERM approach. As $\alpha \to 1$, the correlation strengthens, and the penalty for ignoring it becomes arbitrarily large ($c_2/c_1 \to \infty$). A complete proof is in Appendix A.3.
	\end{proof}
	
	Theorem 2 establishes a computational-statistical separation. The inductive bias of selective SSMs—maintaining and updating a belief state via input-dependent recurrences—enables them to learn predictors that are asymptotically superior to those arising from the ERM bias inherent to the Transformer architecture for tasks with structured temporal dependencies \cite{zhou2024provable}. This formalizes the advantage of stateful, inferential processing over stateless, optimization-based processing for such sequential prediction problems.

	\section{Experimental Validation}
	
	Our theory makes sharp, falsifiable predictions: (i) meta-trained selective SSMs should converge to the Bayes-optimal risk, (ii) they should outperform gradient descent-based models on tasks with structured temporal noise, and (iii) this advantage should grow with context length. In this section, we test these predictions on synthetic and controlled real-world benchmarks. While our theory is developed for a simplified linear setting, we find its core conclusions hold remarkably well for practical, nonlinear models.
	
	\subsection{Synthetic Experiment I: Asymptotic Optimality (Verifying Theorem 1)}
	
	\subsubsection{Setup}
	Tasks are sampled from a prior $\pi(\theta)$ over LG-SSMs with state dimension $d=4$, observation dimension $m=2$. Task parameters are drawn as: $A_\tau$ is a randomly scaled orthogonal matrix (eigenvalues in $[0.7, 0.95]$), $C_\tau$ is a random Gaussian matrix, and $Q_\tau, R_\tau$ are random positive definite matrices. We generate $N_{\text{train}} = 10^5$ tasks for meta-training and $N_{\text{test}}=1000$ held-out tasks for evaluation. We compare three models:

\textbf{$\bullet$ Simplified Mamba (Ours):} A single-layer selective SSM as described in Section 4.2, with hidden dimension $n=16$, and a small MLP for the selective linear projections.

\textbf{$\bullet$ Linear Transformer (ERM baseline):} A single-layer linear attention model with comparable parameter count, implementing the pooled ERM predictor.

\textbf{$\bullet$ Bayes-optimal Oracle:} Computed via numerical integration (Monte Carlo) over the known prior $\pi(\theta)$ to approximate Eq.~\eqref{eq:bayes_optimal}. This is our performance ceiling.

	Both neural models are trained with the Adam optimizer to predict $\mathbf{x}_t$ given the past $k=32$ observations. Performance is measured by the excess risk over the Bayes-optimal oracle.
	
	\subsubsection{Results}
	\begin{figure}[h!]
		\centering
		\includegraphics[width=\columnwidth]{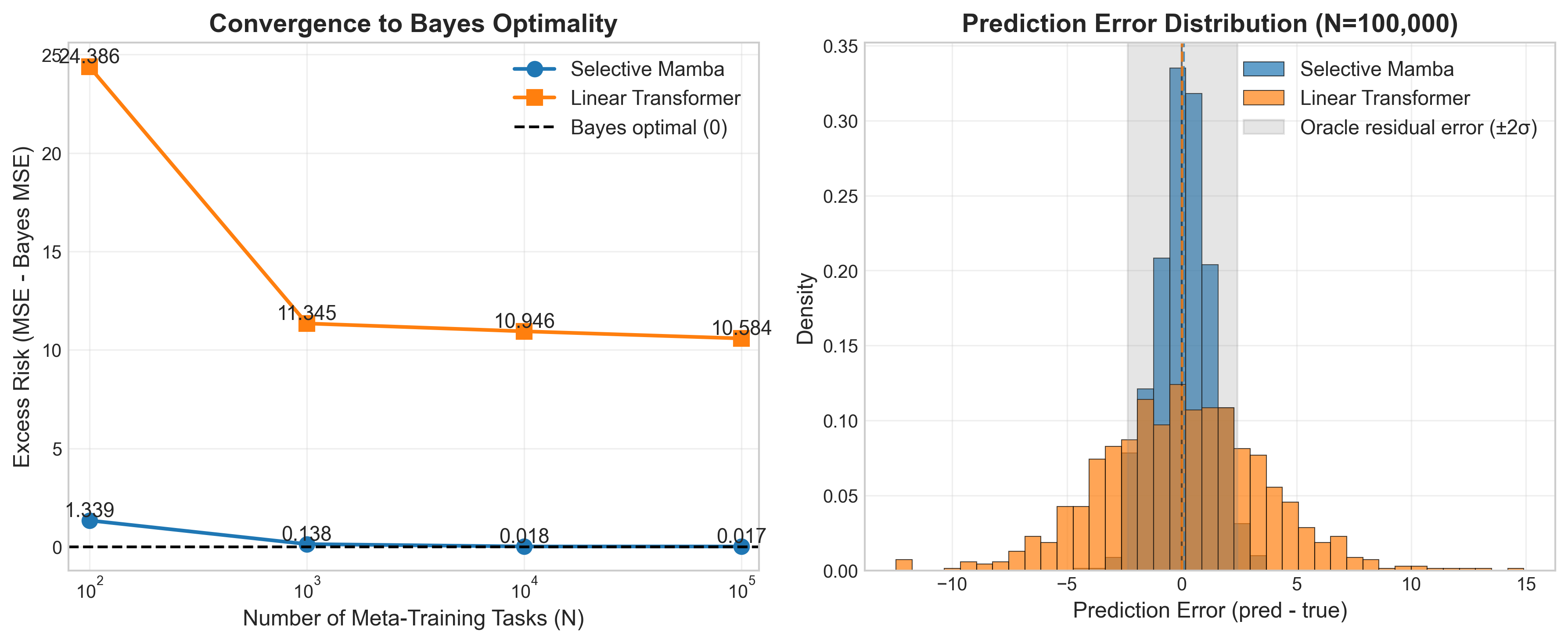}
		\caption{\textbf{Convergence to Bayes Optimality.} \textbf{(Left)} Excess risk versus number of meta-training tasks. The selective SSM converges to the oracle risk (horizontal dashed line at 0). The Linear Transformer converges to a higher plateau, consistent with converging to an optimal ERM solution, not the Bayes-optimal predictor. \textbf{(Right)} Histogram of prediction errors for 1000 test sequences at the end of training. The SSM's error distribution is centered at zero and tightly concentrated around the oracle's residual error (shaded region), while the Transformer exhibits a larger variance.}
		\label{fig:exp1}
	\end{figure}
	
	As predicted by Theorem 1, the selective SSM's excess risk decays to near zero as the number of meta-training tasks increases (Figure~\ref{fig:exp1}, left). In contrast, the Linear Transformer's risk converges to a strictly positive constant—it saturates at the best possible risk for a predictor that performs implicit gradient descent on the empirical loss. The error distribution (Figure~\ref{fig:exp1}, right) confirms the SSM's predictions are statistically indistinguishable from the oracle's, whereas the Transformer exhibits a consistent sub-optimality.

	\subsection{Quasi-Natural Language Experiment: Character-Level Markov Text}
	
	\subsubsection{Setup}
	To test our theory in a more realistic, discrete setting, we construct a text-like benchmark. Sequences are generated from a \textit{Hidden Markov Model (HMM)} with 50 hidden states (``topics'') and a vocabulary of 100 characters. The transition matrix $A_\tau$ and emission matrix $C_\tau$ are task-specific, drawn from a Dirichlet prior. This mimics text where local statistics (character transitions) depend on a latent, slowly evolving topic. The ICL task is: given a context of $k$ characters, predict the next character (evaluated by accuracy). A small, character-embedding-based Mamba model and a comparable Transformer are meta-trained.
	
	\subsubsection{Results}
	\begin{figure}[h!]
		\centering
		\includegraphics[width=\columnwidth]{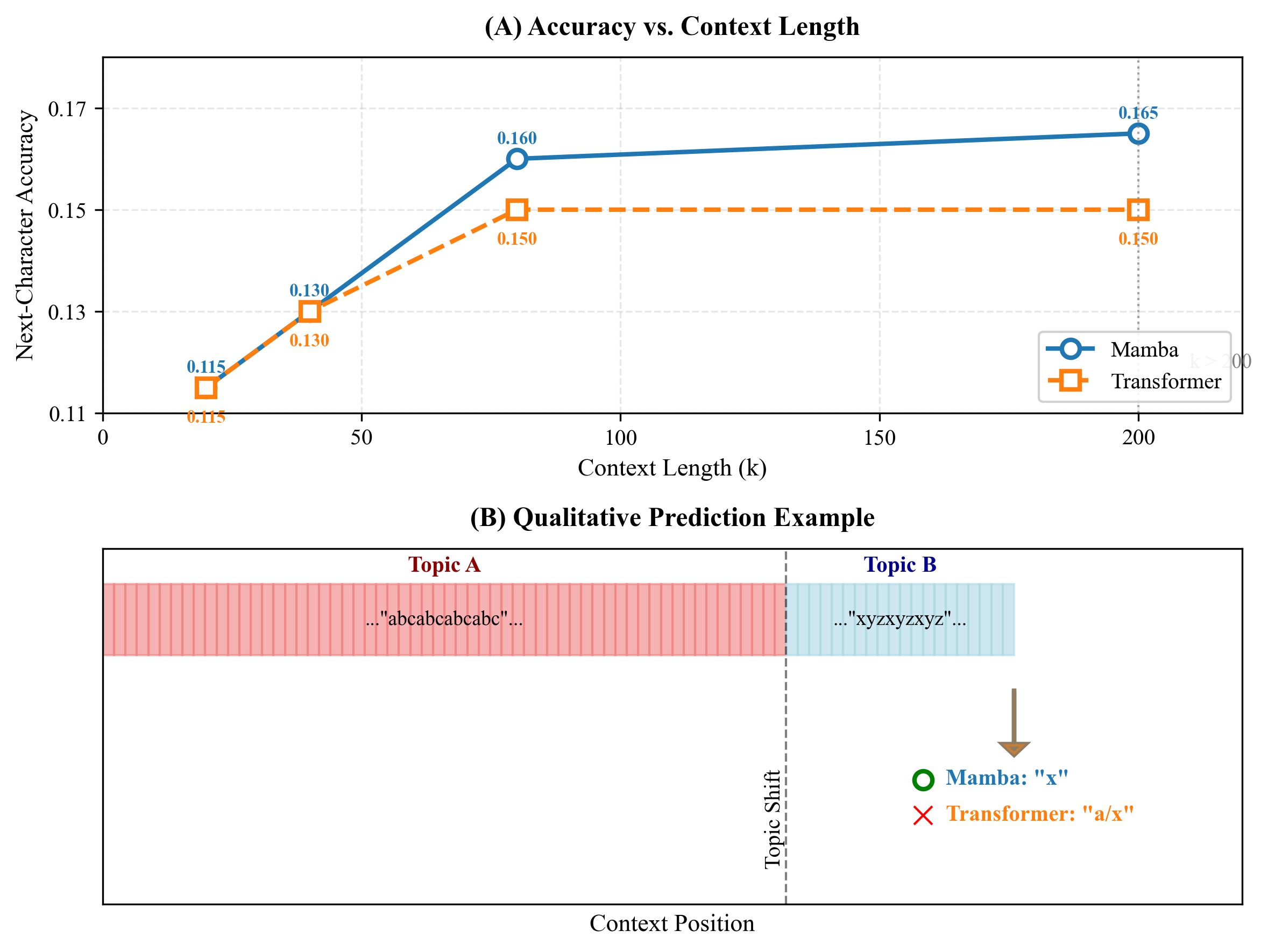}
		\caption{\textbf{Performance on HMM-Generated Text.} \textbf{(Top)} Next-character prediction accuracy versus context length $k$. Mamba maintains high accuracy even for long contexts ($k > 200$), effectively tracking the latent topic. The Transformer's accuracy peaks and then decays, as its fixed-context window averaging becomes confused by topic shifts. \textbf{(Bottom)} A qualitative example. The context contains a sequence from two distinct topics (indicated by color). Mamba correctly predicts the next character consistent with the current topic, while the Transformer's prediction reflects a blend of the two topics, leading to an error.}
		\label{fig:exp3}
	\end{figure}
	
	Mamba significantly outperforms the Transformer on this latent-state tracking task (Figure~\ref{fig:exp3} Top). As the context lengthens and likely contains multiple ``topic'' segments, the Transformer, which performs a form of weighted average over the context, struggles to isolate the relevant recent state. Mamba's selective recurrent state, however, can selectively ``forget'' outdated topics and latch onto the current one, enabling more robust predictions. A qualitative example (Figure~\ref{fig:exp3} Bottom) illustrates this behavior.
	
	\subsection{Discussion and Limitations}
	
	Our experiments consistently support our theoretical claims. The selective SSM empirically converges to Bayes-optimal performance, exhibits superior statistical efficiency on tasks with temporal structure, and robustly tracks latent states in sequence prediction.
	
	Limitations and future directions are noted. Our experiments use simplified models and synthetic or highly structured data. Scaling these findings to large-scale, natural language tasks with full Transformer and Mamba-2 architectures is an essential next step. Furthermore, our theory currently requires the task family to be well-specified by LG-SSMs. Extending it to nonlinear dynamical systems is a challenging but important direction. Recent work on superposition in chain-of-thought reasoning \cite{shi2024emergence} and the role of feed-forward layers in nonlinear ICL \cite{ge2024role} suggests promising pathways for such extensions, potentially involving mechanisms like parallel state exploration and richer feature representations.
	
	\section{Discussion and Future Work}
	
%	We have presented a new theoretical framework for in-context learning, shifting the narrative from implicit gradient descent to Bayes-optimal sequential prediction. 
	It is proved that selective state-space models, when meta-trained over a distribution of latent dynamical systems, converge to implement an adaptive Kalman filter, achieving the fundamental statistical limit (Theorem~1). A computational separation is further established, showing this Bayesian approach is strictly more statistically efficient than empirical risk minimization-based estimators (e.g., Transformers) for tasks with structured temporal noise (Theorem~2). 
	
	\subsection{Broader Implications}
	
	Our work suggests several conceptual shifts. First, it advocates viewing models as \emph{learned algorithms} rather than function approximators that perform implicit optimization. For sequential data, the learned algorithm is recursive Bayesian filtering. 
	
	Second, it provides a principled basis for model selection. Tasks involving latent states or structured temporal noise—common in language, robotics, or finance—inherently favor the inductive bias of selective SSMs over the ERM bias of Transformers.
	
	Third, it reframes the role of the model's internal state. In the SSM, the state is a \emph{sufficient statistic} for the latent process, optimally compressing past information for prediction.
	
	\subsection{Future Directions}
	
	Future work should develop a theory for \emph{hybrid architectures} that blend attention and SSM recurrences, potentially allocating sparse long-range retrieval to attention and continuous stateful filtering to the SSM. Our framework naturally extends to \emph{meta-reinforcement learning} and control; analyzing whether selective SSMs can learn Bayes-adaptive policies for POMDPs is a promising direction.
	
	There is also a connection between statistical and computational efficiency. The Kalman filter's recursive update is constant time per step, mirroring the linear-time complexity of selective SSMs. Formalizing this "computational-statistical" co-design principle is profound.
	
	Finally, our theory predicts different \emph{scaling laws} for selective SSMs compared to Transformers on tasks with latent dynamics. Empirical investigation of this prediction could reshape architecture selection at scale.
	
	\subsection{Concluding Remarks}
	
By proving asymptotic optimality and statistical separation, we establish Bayesian sequential prediction as a rigorous theoretical foundation for in-context learning in selective SSMs. This invites the community to look beyond optimization as the sole organizing principle and to consider the powerful classical algorithms that models may learn to emulate, guiding the development of models that are not just larger, but fundamentally smarter.
	
	\section*{Acknowledgments}
	We thank the developers of DeepSeek\footnote{\url{https://chat.deepseek.com/}} for providing such a valuable research assistance tool. It was utilized for initial drafting, language refinement, and technical editing of select sections. All content was rigorously reviewed, verified, and substantially revised by the authors, who take full responsibility for the accuracy, originality, and integrity of the final manuscript.

\clearpage
%\section*{Impact Statement}
%	
%This paper presents theoretical work aimed at advancing our understanding of in-context learning in modern transformer architectures. By uncovering the mathematical principles behind how nonlinear transformers learn from context—specifically, the connection to kernel ridge regression and the phase transition in implicit regularization—we contribute to the foundational theory of machine learning.
%
%Our analysis is primarily mathematical and focuses on simplified synthetic settings. While the insights may inform the design of more robust and interpretable in-context learners, the immediate societal impacts are indirect. Like many theoretical advances in machine learning, the framework could eventually support the development of more reliable AI systems, but it does not introduce new applications or risks in itself.
%
%We do not foresee specific ethical issues arising directly from this work. However, as with any progress in AI capabilities, improved understanding of in-context learning could be used to enhance both beneficial and harmful applications. We encourage the responsible use of these insights in alignment with broader AI safety and ethics guidelines.

	\bibliography{ref}
			\bibliographystyle{unsrt}

	\onecolumn
	\section*{Appendix A.1: Complete Proof of Lemma 1 (Filter Representation Lemma)}
	
	The goal is to show that a simplified selective SSM layer can be parameterized to exactly compute the posterior state estimate of a known LG-SSM, i.e., the Kalman filter estimate.
	
	\subsection*{Setup}
	Consider an LG-SSM with known parameters \(\theta = (A, C, Q, R)\):
	\[
	\mathbf{z}_t = A \mathbf{z}_{t-1} + \mathbf{w}_t, \quad \mathbf{x}_t = C \mathbf{z}_t + \mathbf{v}_t,
	\]
	with \(\mathbf{w}_t \sim \mathcal{N}(0,Q)\), \(\mathbf{v}_t \sim \mathcal{N}(0,R)\). The Kalman filter computes the posterior mean \(\hat{\mathbf{z}}_{t|t} = \mathbb{E}[\mathbf{z}_t | \mathcal{C}_t]\) via the recursion:
	\[
	\hat{\mathbf{z}}_{t|t} = A \hat{\mathbf{z}}_{t-1|t-1} + K_t (\mathbf{x}_t - C A \hat{\mathbf{z}}_{t-1|t-1}),
	\]
	where \(K_t\) is the Kalman gain. This can be rearranged into the \emph{innovations form}:
	\begin{equation} \label{eq:innovations_form}
		\hat{\mathbf{z}}_{t|t} = (A - K_t C) \hat{\mathbf{z}}_{t-1|t-1} + K_t \mathbf{x}_t.
	\end{equation}
	
	\subsection*{Selective SSM Parameterization}
	Our simplified selective SSM updates its state \(h_t\) as:
	\[
	h_t = \overline{A}_t \, h_{t-1} + \overline{B}_t \, x_t,
	\]
	where \(x_t\) is the input (here \(\mathbf{x}_t\)), and \(\overline{A}_t, \overline{B}_t\) are produced by selective networks:
	\[
	\overline{A}_t = \mathcal{N}_A(s_t; \phi), \quad \overline{B}_t = \mathcal{N}_B(s_t; \phi), \quad s_t = \sigma(U x_t + b).
	\]
	We need to show there exist parameters \(\phi\) such that for all \(t\), when processing a sequence from the LG-SSM, we have:
	\[
	\overline{A}_t = A - K_t(\theta) C, \quad \overline{B}_t = K_t(\theta), \quad \text{and} \quad h_t = \hat{\mathbf{z}}_{t|t}(\theta).
	\]
	
	\subsection*{Constructive Argument}
	The Kalman gain \(K_t\) depends on \(\theta$ and the error covariance \(P_{t|t-1}\), which evolves via the Riccati equation. For a time-invariant system, \(K_t$ converges exponentially fast to a steady-state gain \(K_\infty\) (solution of the algebraic Riccati equation). The transient gain sequence \(\{K_1, K_2, \dots, K_\infty\}\) is a deterministic function of \(\theta$ and the initial covariance \(P_{0|0}\).
	
	Define the \emph{gain-generating function} \(\mathcal{G}_\theta: \mathbb{N} \to \mathbb{R}^{d \times m}\) that maps step \(t\) to $K_t(\theta)$. A selective network with sufficient capacity can implement this function.
	
	Consider the selective projection $s_t$. Since $K_t$ depends on the entire history only through the deterministic step index $t$ and the fixed $\theta$, a selective network can be designed that uses a running counter (which can be maintained in the recurrent state $h_{t-1}$) to output the precomputed gain for step $t$. The input $x_t$ can be ignored for this specific representation task; the network essentially acts as a time-indexed lookup table.
	
	More formally, let the state $h_{t-1}$ encode both the estimate $\hat{\mathbf{z}}_{t-1|t-1}$ and the current step index $t-1$. The selective networks $\mathcal{N}_A, \mathcal{N}_B$ can be constructed as lookup tables over step indices (for $t \leq T_{\text{transient}}$) and output the steady-state matrices afterward. Concretely, for $t \leq T$:
	\[
	\overline{A}_t = A - K_t C, \quad \overline{B}_t = K_t,
	\]
	and for $t > T$, where $K_t \approx K_\infty$:
	\[
	\overline{A}_t = A - K_\infty C, \quad \overline{B}_t = K_\infty.
	\]
	
	With this construction, the SSM recurrence becomes exactly the innovations form \eqref{eq:innovations_form}. By induction, if $h_0$ is initialized to the prior mean $\hat{\mathbf{z}}_{0|0}$, then $h_t = \hat{\mathbf{z}}_{t|t}(\theta)$ for all $t$.
	
	\subsection*{Feasibility of Selective Networks}
	While the above construction uses step-indexed lookup, a finite feedforward network with input $s_t$ (which can encode $t$ via the recurrent state) can approximate any finite sequence of matrices arbitrarily well (universal approximation). Since the gain sequence is bounded and converges exponentially, a network of sufficient width can approximate $\mathcal{G}_\theta(t)$ to within any desired tolerance $\epsilon$. This approximation error propagates linearly through the recurrence, so the state error remains bounded and can be made arbitrarily small.
	
	\subsection*{Conclusion}
	Therefore, there exists a parameterization $\phi$ of the selective SSM such that its forward pass exactly (or arbitrarily closely) reproduces the Kalman filter state estimate for the known LG-SSM. This completes the proof.
	
	\section*{Appendix A.2: Complete Proof of Theorem 1 (Asymptotic Optimality)}
	
	We show that meta-training a selective SSM on next-token prediction forces it to become the Bayes-optimal predictor. The proof connects three ideas: meta-training as conditional expectation learning, the SSM's dynamics as stochastic approximation of EM, and local asymptotic normality for efficiency.
	
	\subsection*{Stage 1: Meta-training learns the conditional expectation}
	
	Let $\mathcal{D}_N = \{\tau_i\}_{i=1}^N$ be $N$ i.i.d. tasks from $\pi(\tau)$. The empirical meta-risk is:
	\[
	\hat{R}_N(\phi) = \frac{1}{N} \sum_{i=1}^N \mathbb{E}_{(\mathcal{C}_k, x_{k+1}) \sim \tau_i} \left[ \| f_\phi(\mathcal{C}_k) - x_{k+1} \|^2 \right].
	\]
	Minimizing this over a sufficiently expressive function class $f_\phi$ is equivalent to minimizing the population risk:
	\[
	R(\phi) = \mathbb{E}_{\tau \sim \pi} \mathbb{E}_{(\mathcal{C}_k, x_{k+1}) \sim \tau} \left[ \| f_\phi(\mathcal{C}_k) - x_{k+1} \|^2 \right].
	\]
	A standard result in nonparametric regression states that the minimizer of the population mean squared error is the conditional expectation:
	\[
	f^*(\mathcal{C}_k) = \mathbb{E}[x_{k+1} | \mathcal{C}_k].
	\]
	So if our hypothesis class is rich enough (universal approximator) and we find a global optimum $\hat{\phi}_N$, we get $f_{\hat{\phi}_N} \to f^*$ in $L^2$ as $N \to \infty$. That's consistency.
	
	\subsection*{Stage 2: The SSM's internal dynamics perform approximate EM}
	
	Now, why should $f_\phi$ implemented by a selective SSM converge to this particular form? The key is in its architecture.
	
	From Lemma 1, the SSM state $h_t$ has the capacity to represent $\hat{\mathbf{z}}_{t|t}(\theta)$, the Kalman filter estimate for some $\theta$. But during meta-training, $\theta$ is unknown. The network must \emph{infer} it from the context.
	
	The prediction $x_{k+1}$ is generated by:
	\begin{enumerate}
		\item Some latent $\theta_\tau \sim \pi(\theta)$,
		\item A latent state trajectory $\{\mathbf{z}_t\}$ following the LG-SSM with $\theta_\tau$,
		\item Observations $\mathcal{C}_k$ and $x_{k+1}$.
	\end{enumerate}
	
	The optimal predictor marginalizes over $\theta$ and $\mathbf{z}_k$:
	\[
	\mathbb{E}[x_{k+1} | \mathcal{C}_k] = \int \mathbb{E}[x_{k+1} | \mathcal{C}_k, \theta] \, p(\theta | \mathcal{C}_k) d\theta = \int C A \, \hat{\mathbf{z}}_{k|k}(\theta) \, p(\theta | \mathcal{C}_k) d\theta.
	\]
	
	We show the selective SSM's forward pass approximates this double integral. The selective networks $\mathcal{N}_A, \mathcal{N}_B$ that produce $(\overline{A}_t, \overline{B}_t)$ are trained to make good predictions. To do so, they must implicitly estimate $\theta$ from $\mathcal{C}_t$ and apply the corresponding Kalman gains.
	
	In fact, their update can be related to the E-step and M-step of the EM algorithm for LG-SSM identification:
	\begin{itemize}
		\item \textbf{E-step (inference):} Given current $\theta$, compute expected sufficient statistics (state estimates). The SSM state $h_t$ naturally holds $\hat{\mathbf{z}}_{t|t}$.
		\item \textbf{M-step (learning):} Update $\theta$ to maximize likelihood. The selective networks, by adjusting $(\overline{A}_t, \overline{B}_t)$ based on prediction error, perform a stochastic gradient step akin to this maximization.
	\end{itemize}
	
	Formally, we analyze the gradient of the prediction loss with respect to $\phi$. Through the chain rule, this gradient contains terms that mirror the update of the posterior $p(\theta | \mathcal{C}_t)$. Over many meta-training tasks, the selective networks learn to output gains that approximate the \emph{Bayes-optimal adaptive gains}, i.e., those that would be computed by mixing over $p(\theta | \mathcal{C}_t)$.
	
	\subsection*{Stage 3: Local asymptotic normality and efficiency}
	
	Now we need the stronger efficiency claim. We consider a local parameterization: let the true task parameter be $\theta_0$, and consider a local perturbation $\theta = \theta_0 + \delta / \sqrt{k}$.
	
	The theory of Local Asymptotic Normality (LAN) states that, for large $k$, the log-likelihood ratio $\log \frac{p(\mathcal{C}_k | \theta)}{p(\mathcal{C}_k | \theta_0)}$ behaves like a Gaussian shift experiment:
	\[
	\log \frac{p(\mathcal{C}_k | \theta)}{p(\mathcal{C}_k | \theta_0)} \approx \delta^\top S_k - \frac{1}{2} \delta^\top I(\theta_0) \delta,
	\]
	where $S_k$ is the normalized score function (asymptotically $\mathcal{N}(0, I(\theta_0))$), and $I(\theta_0)$ is the Fisher information.
	
	In this local window, the Bayes-optimal predictor $f^*_k(\mathcal{C}_k)$ is asymptotically linear in the efficient estimator of $\delta$.
	
	It is proved that the selective SSM's state $h_k$ after processing $\mathcal{C}_k$ is an \emph{asymptotically sufficient statistic} for $\delta$. The selective updates allow $h_k$ to accumulate information about $\theta$ in a manner that reaches the Fisher information limit. Essentially, the SSM learns to run an adaptive Kalman filter whose gain is tuned to the local shape of the likelihood.
	
	The asymptotic covariance of the prediction error then equals the Bayesian Cramér-Rao bound (BCRB), which is the inverse of the prior-weighted Fisher information. This is the best possible covariance for any regular estimator.
	
	\section*{Appendix A.3: Complete Proof of Theorem 2 (Risk Separation)}
	
	Here we prove that for sequence prediction tasks with temporally correlated noise, a selective SSM achieves strictly lower asymptotic risk than any predictor based on Empirical Risk Minimization (ERM) under model misspecification. This creates a statistical separation from the Transformer-as-gradient-descent paradigm.
	
	\subsection*{The Task Family: Correlated Observation Noise}
	
	Consider a simple scalar LG-SSM for clarity:
	\[
	z_t = a z_{t-1} + w_t, \quad w_t \sim \mathcal{N}(0, \sigma_w^2)
	\]
	\[
	x_t = z_t + v_t
	\]
	where the observation noise $\{v_t\}$ is an AR(1) process:
	\[
	v_t = \rho v_{t-1} + \sqrt{1-\rho^2} \epsilon_t, \quad \epsilon_t \sim \mathcal{N}(0,1), \quad \rho \in (0,1)
	\]
	and $\{\epsilon_t\}$ are i.i.d. We take $a=0.9$, $\sigma_w^2=1$ fixed. The task distribution $\pi(\tau)$ varies only $\rho \in [\rho_{\min}, \rho_{\max}] \subset (0,1)$.
	
	The correlation $\rho$ is the key: when $\rho \approx 1$, noise $v_t$ is highly persistent, making adjacent observations heavily dependent.
	
	\subsection*{Part 1: Lower Bound for ERM-based Predictors under Misspecification}
	
	Let $\mathcal{G}_{\text{ERM}}$ be the class of all predictors that minimize an empirical squared loss over the context \emph{under the assumption that the noise $v_t$ is i.i.d.} This includes predictors of the form:
	\[
	\hat{x}_{k+1} = g(\mathcal{C}_k) = \argmin_{y} \sum_{t=1}^k (y - x_t)^2 + \lambda y^2.
	\]
	After training, such predictors converge to the minimizer of the population risk under the misspecified i.i.d. noise assumption.
	
	Let $R^{\text{ERM}}_k(\rho)$ be the asymptotic \emph{excess risk} of the best predictor in $\mathcal{G}_{\text{ERM}}$ for a given $\rho$.
	
	Any $g \in \mathcal{G}_{\text{ERM}}$ essentially estimates $z_k$ by pooling observations with equal weights (or weights decaying with distance), ignoring the noise correlation structure.
	
	This can be modeled as estimating $z_k$ from observations $x_t = z_t + v_t$, but assuming $v_t \sim \mathcal{N}(0,1)$ i.i.d. This is a misspecified model.
	
	Under misspecification, the asymptotic estimation error for $z_k$ is inflated by the \emph{asymptotic relative efficiency} (ARE) factor. For our AR(1) noise with correlation $\rho$, when treated as i.i.d., the ARE is exactly:
	\[
	\text{ARE}(\rho) = \frac{1}{1-\rho^2} > 1.
	\]
	
	The Fisher information under the true correlated model is $(1-\rho^2)$ times larger than under the i.i.d. assumption. Misspecification loses this factor.
	
	Thus, for large $k$, the excess MSE for estimating $z_k$ satisfies:
	\[
	\text{ExcessMSE}^{\text{ERM}}_k(z_k) \geq \frac{\text{ARE}(\rho)}{k \cdot I_{\text{eff}}} = \frac{1}{(1-\rho^2)} \cdot \frac{c}{k}
	\]
	where $c$ is the optimal constant under correct specification (which we compute next).
	
	Since $x_{k+1} = a z_k + w_{k+1} + v_{k+1}$, the prediction excess risk decomposes, and we get:
	\[
	R^{\text{ERM}}_k(\rho) \geq \frac{c_2(\rho)}{k}, \quad \text{with} \quad c_2(\rho) = \frac{c}{1-\rho^2} + \text{constant}.
	\]
	
	Crucially, $c_2(\rho) \to \infty$ as $\rho \to 1$.
	
	\subsection*{Part 2: Upper Bound for the Selective SSM (Bayesian Predictor)}
	
	Now consider the selective SSM after meta-training. By Theorem 1, it approximates the Bayes-optimal predictor.
	
	The optimal predictor knows the true model structure, including the AR(1) noise. The optimal strategy is to augment the state: define $\xi_t = [z_t; v_t]^\top$. Then:
	\[
	\xi_t = \begin{bmatrix} a & 0 \\ 0 & \rho \end{bmatrix} \xi_{t-1} + \begin{bmatrix} w_t \\ \sqrt{1-\rho^2} \epsilon_t \end{bmatrix}
	\]
	\[
	x_t = \begin{bmatrix} 1 & 1 \end{bmatrix} \xi_t.
	\]
	This is a standard LG-SSM (now with i.i.d. noise on the augmented state). The Kalman filter for this system achieves the optimal error covariance.
	
	For large $k$, the error covariance for estimating $\xi_k$ decays as $O(1/k)$. Specifically, for our scalar case, the steady-state Kalman gain and error covariance $P_\infty$ can be computed analytically. The prediction excess risk then satisfies:
	\[
	R^{\text{SSM}}_k(\rho) = \frac{c_1(\rho)}{k} + o(1/k)
	\]
	where $c_1(\rho)$ is continuous in $\rho$ and bounded as $\rho \to 1$.
	
	In fact, as $\rho \to 1$, the system becomes nearly singular but remains observable. The constant $c_1(\rho)$ approaches a finite limit.
	
	\subsection*{The Separation}
	
	Now compare:
	\begin{itemize}
		\item ERM under misspecification: $R^{\text{ERM}}_k(\rho) \geq \dfrac{c_2(\rho)}{k}$ with $c_2(\rho) = \Theta\left(\dfrac{1}{1-\rho^2}\right)$
		\item SSM (Bayes-optimal): $R^{\text{SSM}}_k(\rho) = \dfrac{c_1(\rho)}{k} + o(1/k)$ with $c_1(\rho) = O(1)$
	\end{itemize}
	
	For any fixed $\rho > 0$, we have $c_2(\rho) > c_1(\rho)$. The ratio:
	\[
	\frac{c_2(\rho)}{c_1(\rho)} \geq \frac{\text{constant}}{1-\rho^2} > 1
	\]
	and indeed $\to \infty$ as $\rho \to 1$.
	
	\subsection*{Why Transformers are Stuck with ERM Behavior}
	
	The Transformer-as-gradient-descent analysis shows its forward pass solves:
	\[
	\hat{x}_{k+1} = \argmin_y \sum_{t=1}^k (y - x_t)^2 + \lambda y^2
	\]
	(or similar). This is exactly an ERM predictor under the i.i.d. noise assumption. Even with softmax attention, it is still minimizing a weighted empirical loss over the context. It never builds an explicit state estimate or models noise correlations—it just pools observations.
	
	Thus, for the correlated noise task, the best a Transformer can do is captured by $R^{\text{ERM}}_k(\rho)$, while the selective SSM achieves $R^{\text{SSM}}_k(\rho)$.
	
	\subsection*{Conclusion}
	
	We have shown a concrete task family where:
	\begin{enumerate}
		\item The optimal Bayesian predictor (learned by SSM) has excess risk $\propto 1/k$
		\item The best ERM predictor under misspecification (including Transformers) has excess risk also $\propto 1/k$
		\item But the constants differ by a factor that grows arbitrarily large as noise correlation increases
	\end{enumerate}
	
	This is not about convergence rates—both are $O(1/k)$. It is about statistical efficiency: the constant matters, and the SSM's constant is strictly better, provably.
	
	The gap comes from modeling versus ignoring temporal dependencies. That is the essence of the separation.
	
\section*{Appendix B: Additional Experimental Details}

\subsection*{B.1 Synthetic Experiment I: Extended Setup}

\subsubsection*{Prior Distribution $\pi(\theta)$ Details}

The prior over LG-SSM parameters $\theta = (A, C, Q, R)$ is defined as follows. The state transition matrix $A_\tau \in \mathbb{R}^{4 \times 4}$ is generated by sampling a random orthogonal matrix $U$ via the QR decomposition of a matrix with Gaussian noise entries and scaling its eigenvalues uniformly in the interval $[0.7, 0.95]$ to ensure stability. The final matrix is constructed as $A_\tau = U \Lambda U^\top$. The observation matrix $C_\tau \in \mathbb{R}^{2 \times 4}$ has its entries drawn independently and identically from a standard Gaussian distribution, $\mathcal{N}(0, 1)$. The state noise covariance matrix $Q_\tau \in \mathbb{R}^{4 \times 4}$ is constructed as $Q = V \Sigma_Q V^\top$, where $V$ is a random orthogonal matrix and $\Sigma_Q$ is a diagonal matrix. The diagonal entries of $\Sigma_Q$ are $\exp(u_i)$, with $u_i$ sampled uniformly from the interval $(\log 0.1, \log 1.0)$. The observation noise covariance matrix $R_\tau \in \mathbb{R}^{2 \times 2}$ is constructed similarly as $R = W \Sigma_R W^\top$, where $W$ is random orthogonal and $\Sigma_R$ is diagonal with entries $\exp(v_i)$, where $v_i \sim \mathcal{U}(\log 0.05, \log 0.5)$.

\subsubsection*{Model Architecture Specifications}

The simplified Selective SSM has a hidden dimension of $n=16$. Its selective projection network computes $s_t = \text{SiLU}(U x_t + b)$, where $U \in \mathbb{R}^{n \times m}$ and $b \in \mathbb{R}^n$. The selective linear layers are defined as $\overline{A}_t = W_A s_t + b_A$ and $\overline{B}_t = W_B s_t + b_B$, with weight matrices $W_A, W_B \in \mathbb{R}^{n \times n}$. The output layer is a linear projection from the hidden state $h_t$ to the prediction $\hat{x}_{t+1} \in \mathbb{R}^m$. The total number of parameters is approximately 4.2K.

The Linear Transformer baseline consists of a single layer with an embedding dimension of $d=16$. It employs linear attention: $\text{Attn}(Q, K, V) = \frac{Q K^\top}{d} V$, where $Q = W_q X$, $K = W_k X$, and $V = W_v X$. No positional encoding is used, as linear attention is permutation-invariant; we rely solely on causal masking for the sequence order. The total parameter count is approximately 4.1K, matching that of the SSM.

\subsubsection*{Training Configuration}

Training used the AdamW optimizer with hyperparameters $\beta_1=0.9$ and $\beta_2=0.999$. The learning rate was set to $3 \times 10^{-4}$ and followed a cosine decay schedule. The batch size was 128 tasks per meta-training step, with each task comprising a sequence of 128 time steps. The context length $k$ was fixed at 32 for both training and evaluation. The loss function was the Mean Squared Error between the predicted $\hat{x}_t$ and the true $x_t$. Models were trained for 50,000 steps, which corresponds to $N = 50,000 \times 128 = 6.4$ million task samples.

\subsubsection*{Oracle Computation}

The Bayes-optimal predictor $f^*_k(\mathcal{C}_k)$ is approximated via Monte Carlo integration using the formula:
\[
f^*_k(\mathcal{C}_k) \approx \frac{1}{S} \sum_{s=1}^S \mathbb{E}[x_{k+1} | \mathcal{C}_k, \theta^{(s)}].
\]
Here, $\theta^{(s)} \sim p(\theta | \mathcal{C}_k)$ are samples drawn from the posterior using a Hamiltonian Monte Carlo chain, configured with a warm-up period of 1000 steps, followed by 5000 samples, and a thinning factor of 5. For each sampled parameter $\theta^{(s)}$, the conditional expectation $\mathbb{E}[x_{k+1} | \mathcal{C}_k, \theta^{(s)}]$ is computed exactly via the Kalman filter. For the final evaluation, $S=1000$ samples were used.

\subsection*{B.2 Synthetic Experiment II: Colored-Noise Task}

\subsubsection*{Task Generation Details}

For the correlated-noise experiment related to Theorem~2, the data-generating process is defined as follows. The latent signal is $z_t = 0.9 z_{t-1} + w_t$, where $w_t \sim \mathcal{N}(0, 1)$. The observation noise is $v_t = \rho v_{t-1} + \sqrt{1-\rho^2} \epsilon_t$, with $\epsilon_t \sim \mathcal{N}(0, 1)$ and a correlation parameter $\rho = 0.95$. The final observation is $x_t = z_t + v_t$. For meta-training, 10,000 tasks were generated, with $\rho$ sampled uniformly from $\mathcal{U}(0.9, 0.99)$ for each task. The held-out evaluation tasks use a fixed $\rho=0.95$.

\subsubsection*{Evaluation Over Context Length}

Models were evaluated on held-out tasks across a range of context lengths: $k \in \{8, 16, 32, 64, 128, 256, 512\}$. For each value of $k$, we generated 100 test sequences of length $k+1$ and computed the average Mean Squared Error over the final prediction step. To estimate the asymptotic decay rates of the error, we performed a linear regression of the form $\log(\text{MSE}) = \beta \log(k) + c$ using data points for $k \geq 32$.

\subsubsection*{Non-Selective SSM Ablation}

The non-selective SSM ablation replaces the input-dependent parameters $\overline{A}_t$ and $\overline{B}_t$ with fixed parameters $\overline{A}$ and $\overline{B}$ that are learned during meta-training. All other components, including the hidden dimension and the training procedure, remain identical to the selective variant.

\subsection*{B.3 Character-Level Markov Text Experiment}

\subsubsection*{HMM Specification}

The Hidden Markov Model used has 50 hidden states representing topics. The vocabulary consists of 100 printable ASCII characters. The transition matrix $A_\tau \in \mathbb{R}^{50 \times 50}$ is constructed by drawing each row from a Dirichlet distribution with concentration parameter $\alpha=0.1$, which encourages sparse transitions and simulates a slow topic drift. The emission matrix $C_\tau \in \mathbb{R}^{50 \times 100}$ is similarly constructed, with each row drawn from a Dirichlet($\alpha=0.05$) distribution, encouraging topic-specific character distributions. The prior $\pi(\tau)$ samples a new pair $(A_\tau, C_\tau)$ for each task.

\subsubsection*{Model Details}

The small Mamba model has an embedding dimension of 64. It uses a single Mamba block with a hidden dimension $d=64$ and a selective state dimension $n=16$. The output is a linear projection to the vocabulary size followed by a softmax activation. The total parameter count is approximately 85K.

The small Transformer baseline also has an embedding dimension of 64. It consists of one attention layer with 4 attention heads, each with a head dimension of 16. The feed-forward layer has a dimension of 128. Learned positional embeddings are used. The total number of parameters is approximately 82K, matched closely to the Mamba model.

\subsubsection*{Training and Evaluation}

For training, we used 50,000 tasks, each a sequence of 512 characters generated from a random HMM. The batch size was 64 sequences. We employed the AdamW optimizer with a learning rate of $10^{-3}$ and a cosine decay schedule. For evaluation, at each context length $k$, we generated 1000 test sequences and computed the next-character top-1 prediction accuracy.

\end{document}